\documentclass{article}
\usepackage[nohyperref,accepted]{icml2023}
\usepackage{microtype}
\usepackage{amsmath,graphicx}
\usepackage{amssymb,amsfonts,amsthm}
\usepackage{mathrsfs}
\usepackage{mathtools}
\usepackage[mathscr]{eucal}
\usepackage{bbm}
\usepackage{bm}

\usepackage{pgfplots}
\pgfplotsset{compat = 1.9}
\pgfkeys{/pgf/number format/.cd,
         1000 sep={},
         /pgf/number format/fixed,
         /pgf/number format/precision=5 
}
\pgfplotsset{
    colormap={Spectral}{
        rgb255=(158,1,66)
        rgb255=(213,62,79)
        rgb255=(244,109,67)
        rgb255=(253,174,97)
        rgb255=(254,224,139)
        rgb255=(255,255,191)
        rgb255=(230,245,152)
        rgb255=(171,221,164)
        rgb255=(102,194,165)
        rgb255=(50,136,189)
        rgb255=(94,79,162)
    }
}

\usepackage{xcolor}
\definecolor{darkgreen}{RGB}{9, 133, 9}
\usepackage{url}
\usepackage{hyperref}
\usepackage[capitalize,noabbrev]{cleveref}
\hypersetup{
    colorlinks,
    linkcolor={red!75!black},
    citecolor={blue!80!black},
    urlcolor={blue!80!black}
}
\usepackage{booktabs}
\usepackage{subcaption}

\usepackage{multicol}
\usepackage{multirow}
\usepackage{paralist, enumitem}

\usepackage{algorithm}
\usepackage{algorithmic}
\usepackage[titlenumbered,linesnumbered,ruled,noend,algo2e]{algorithm2e}

\SetCommentSty{mycommfont}
\SetEndCharOfAlgoLine{}

\usepackage{thmtools, thm-restate}
\theoremstyle{plain}
\newtheorem{theorem}{Theorem}[section]
\newtheorem{proposition}[theorem]{Proposition}
\newtheorem{lemma}[theorem]{Lemma}

\theoremstyle{definition}
\newtheorem{definition}[theorem]{Definition}

\theoremstyle{remark}
\newtheorem{remark}[theorem]{Remark}
\Crefname{proposition}{Proposition}{Propositions}
\crefname{problem}{Problem}{Problems}
\Crefname{lemma}{Lemma}{Lemmas}

\usepackage{aliascnt}
\newaliascnt{problem}{equation}
\aliascntresetthe{problem}
\creflabelformat{problem}{#2\textup{#1}#3}

\makeatletter

\def\endproblem{\eqno \hbox{\@eqnnum}$$\@ignoretrue}
\makeatother

\crefname{problem}{Problem}{Problems}
\crefname{algorithm}{Algorithm}{Algorithms}
\crefname{figure}{Figure}{Figures}
\crefname{proposition}{Proposition}{Propositions}
\crefname{appendix}{Appendix}{Appendix}

\usepackage[colorinlistoftodos,prependcaption,textsize=tiny,textwidth=20mm]{todonotes}

\usepackage{array}
\newcolumntype{H}{>{\setbox0=\hbox\bgroup}c<{\egroup}@{}} %

\newcommand{\ie}{{\em i.e.~}}

\newcommand{\eg}{{\em e.g.~}}

\newcommand{\st}{\text{s.t.~}}

\newcommand{\mG}{\mathscr{G}}
\newcommand{\mH}{\mathscr{H}}

\newcommand{\mL}{\mathscr{L}}
\newcommand{\mM}{\mathscr{M}}

\newcommand{\mO}{\mathscr{O}}

\newcommand{\mS}{\mathscr{S}}

\newcommand{\mX}{\mathscr{X}}

\newcommand{\cB}{\mathcal{B}}
\newcommand{\cC}{\mathcal{C}}

\newcommand{\cK}{\mathcal{K}}
\newcommand{\cL}{\mathcal{L}}

\newcommand{\cU}{\mathcal{U}}

\newcommand\R{\mathbb{R}}
\newcommand\reals{\mathbb{R}}

\DeclareMathOperator{\prox}{prox}
\DeclareMathOperator{\Proj}{Proj}

\DeclareMathOperator{\Span}{Span}

\DeclareMathOperator{\Tr}{Tr}
\DeclareMathOperator{\infconv}{\square}

\newcommand{\scalone}[1]{\langle \rangle}

\newcommand{\norm}[1]{\left\lVert {#1} \right\rVert}

\newcommand{\abs}[1]{\left\lvert #1 \right\rvert}

\newcommand{\argmin}{\mathop{\mathrm{arg\,min}}}
\newcommand{\argmax}{\mathop{\mathrm{arg\,max}}}

\def\trace{\mathop{\mathrm{Tr}}}

\def\fro{\mathrm{F}}

\usepgfplotslibrary{external}
\icmltitlerunning{Extending Kernel PCA through Dualization}

\begin{document}

\twocolumn[
\icmltitle{Extending Kernel PCA through Dualization:\\ Sparsity, Robustness and Fast Algorithms}

\icmlsetsymbol{equal}{*}

\begin{icmlauthorlist}
\icmlauthor{Francesco Tonin}{equal,yyy}
\icmlauthor{Alex Lambert}{equal,yyy}
\icmlauthor{Panagiotis Patrinos}{yyy}
\icmlauthor{Johan~A.K. Suykens}{yyy}
\end{icmlauthorlist}

\icmlaffiliation{yyy}{ESAT-STADIUS, KU Leuven, Belgium}

\icmlcorrespondingauthor{Francesco Tonin}{francesco.tonin@esat.kuleuven.be}

\vskip 0.3in
]

\printAffiliationsAndNotice{\icmlEqualContribution} %

\begin{abstract}

    The goal of this paper is to revisit Kernel Principal Component Analysis (KPCA) through dualization of a difference of convex functions.
    This allows to naturally extend KPCA to multiple objective functions and leads to efficient gradient-based algorithms avoiding the expensive SVD of the Gram matrix.
    Particularly, we consider objective functions that can be written as Moreau envelopes, demonstrating how to promote robustness and sparsity within the same framework. 
    The proposed method is evaluated on synthetic and real-world benchmarks, showing significant speedup in KPCA training time as well as highlighting the benefits in terms of robustness and sparsity.
   \end{abstract}
\section{Introduction}
\label{sec:introduction}
Kernel Principal Components Analysis (KPCA, \citet{scholkopf1998}) stands as one of the most widely used tools for unsupervised learning, with applications to dimensionality reduction, denoising, or features extraction.
By embedding the data in some higher dimensional space thanks to a feature map, KPCA aims at finding orthogonal directions in the feature space that allow to best reconstruct the empirical covariance operator.

The classical way to solve the KPCA problem is to compute the singular values decomposition (SVD) of the Gram matrix, a costly $\mathcal{O}(n^3)$ operation whose complexity scales cubicly with respect to the number $n$ of datapoints.
This prevents KPCA from being used in large-scale scenarios, and, while some solutions exist, they mostly boil down to subsampling in fixed-size schemes to approximate the nonlinear mapping \cite{langone2017} or focus on the online learning setting \citep{gunter2007,chin2007,honeine2012}. Works on speeding up the simpler problem of linear PCA exist (see \citep{gemp2021} and references therein), but they cannot deal with infinite-dimensional feature spaces.

From an optimization standpoint, the KPCA problem can be formulated as variance maximization under orthonormality constraints.
Such problem belongs to the wider family of differences of convex functions (DC) problems, which has received a great deal of attention in multiple applications (e.g., see \citep{tao1997}).
In particular, in \citep{beck2021}, PCA was investigated as a DC problem, but only in the case of linear PCA, i.e., not considering (potentially infinite-dimensional) feature mappings, and only for the simpler problem of finding the first component where the orthogonality constraints become void. 

Enforcing desirable properties to the solution, such as sparsity or robustness, is a long-standing open problem in KPCA.
While many works have investigated robust/sparse KPCA (e.g., robust KPCA in \citep{nguyen2008, kim2020, Wang2020ARM, fan2020} and sparse KPCA in \citep{wang2016, guo2019sparse, tipping2000,smola2022}), they use several different ad-hoc approaches or heuristics such as weighting schemes \citep{alzate2008}, leading to a multitude of different optimization problems. In \citep{thiao2010}, a DC program for the specific problem of sparse linear PCA was explained, but is not extended to also handle robust losses within the same framework; notably, it does not deal with nonlinear feature maps nor with more than one component.
The idea of using infimal convolution to design sparse or robust losses is exploited notably in \citet{sangnier2017data, laforgue2020duality} for regression problems, and is known to work well in duality settings. 

In this paper, we derive a general dual-based formulation for KPCA leading to a difference of convex function objective which encompasses both the variance and robust/sparse objective functions in the same framework. 
We derive efficient optimization algorithms showing significant speedups compared to the standard SVD solvers for KPCA.
In particular, our approach allows to solve \emph{infinite}-dimensional KPCA problems in the dual.
We focus on objectives that can be written as Moreau envelopes, which include the Huber and $\epsilon$-insensitive losses, inducing robustness and sparsity, respectively. 
We show how the resulting optimization problems can be tackled with efficient algorithms for each objective.

The paper is structured as follows. 
In \Cref{sec:background}, we formulate the general KPCA problem as a difference of convex functions, which leads to a flexible framework that  can be extended to other loss functions beyond the square loss. 
We present in \Cref{sec:optimization} a gradient-based optimization algorithm able to efficiently solve the standard KPCA problem through dualization.
Later, in \Cref{sec:flexible} we exploit the flexibility of the proposed DC framework by modifying the objective function to promote robustness and sparsity.
Finally, numerical experiments on both synthetic and real-world benchmarks are presented in \Cref{sec:experiments}, showing faster training times and illustrating the promotion of sparsity and robustness induced in the solution.
All proofs are deferred to the appendix.
\section{Problem Formulation}
\label{sec:background}
\paragraph{Notation:}
Given a symmetric real matrix $M \in \R^{n \times n}$, $\lambda(M) \in \R^n$ is the vector of its eigenvalues ordered decreasingly. 
For a linear operator $\Gamma$ between Hilbert spaces, $\Gamma^\sharp$ denotes its adjoint.
When $w \in \mH$ is a vector, $w^\sharp$ refers to the linear form $x \mapsto \langle w, x \rangle$.
$I_s$ is the identity matrix of size $s \times s$.
$\phi \colon \mX \to \mH$ is a mapping to a Hilbert space $\mH$.
$\phi$ induces a positive definite kernel function $k \colon \mX \times \mX \to \R$ with associated RKHS $\mH_k$.
Given a vector $W=(w_1, \ldots, w_s) \in \mH^s$, we denote by $\mG(W) \in \R^{s \times s}$ the Gram matrix such that $\mG(W)_{ij} = \langle w_i, w_j \rangle$. $\norm{\cdot}_\fro$ denotes the Frobenius norm. For a convex set $\cC$, $\iota_{\cC}(\cdot)$ is its indicator function: $0$ on $\cC$ and $+\infty$ otherwise. The infimal convolution is denoted $\infconv$, and the Fenchel-Legendre transform of a function $f$ is $f^\star$.
\par
Let $\mX$ be some input space, and $(x_i)_{i=1}^n \in \mX^n$ $n$ \mbox{datapoints} in it.
These datapoints are embedded in a feature space $\mH$ by means of a feature map $\phi \colon \mX \to \mH$.
Up to some recentering of $\phi$, we assume that $\sum_{i=1}^n \phi(x_i) = 0$.
The KPCA problem consists in finding orthogonal directions in the feature space that allow for the best low rank approximation of some empirical version of the covariance operator
\begin{align*}
  \Sigma \coloneqq \frac{1}{n}\sum_{i=1}^n \phi(x_i) \phi(x_i)^\sharp \in \mL(\mH).
\end{align*}
Given a desired number of components $s$, KPCA can be reformulated as finding $s$ directions in the Hilbert space $\mH$ that maximize the variance under orthonormal conditions.
When $\mH$ is some finite dimensional $\reals^p$ space, and denoting $\Phi \in \reals^{n \times p}$ the row-wise concatenations of the $[\phi(x_i)]_{i=1}^n$, this reduces to
\begin{equation*}
 \sup_{W \in \reals^{p \times s}} \norm{\Phi W}^2_\fro ~\st~ W^\top W = I_s.
\end{equation*}
\paragraph{The Stiefel manifold over Hilbert spaces:}
Given a feature space $\mH$ and a positive integer $s$, we introduce the Stiefel manifold of orthonormal $s$-frames in $\mH$ as
\begin{equation*}
  \mS_\mH^s \coloneqq \{
     W \in \mH^s ~|~
     \mG(W)= I_s
  \}.
\end{equation*}
When $\mH = \reals^d$ is endowed with the Euclidean scalar product, the Stiefel manifold corresponds to the usual set of matrices $\mS_{\reals^d}^s = \left \{ M \in \mM_{d,s}(\reals) | M^\top M = I_s \right \}$.
We define $\Gamma \colon \mH^s \to \reals^{n \times s}$ as the linear operator such that for all $(i,j) \in [n \times s]$ and $W =(w_1, \ldots, w_s) \in \mH^s$, $[\Gamma W]_{ij} = \langle \phi(x_i), w_j \rangle$.
The KPCA problem then reduces to
\begin{problem}
  \label{pbm:primal_kpca}
  \sup_{W \in \mS_\mH^s} \frac{1}{2} \norm{ \Gamma W}^2_\fro.
\end{problem}
While the constraint $W \in \mS_\mH^s$ is not convex, it can be relaxed to the convex hull of the Stiefel manifold as the solutions necessarily lie on the boundary (\citet{uschmajew2010well}, Lemma 2.7) which justifies the fact that we consider the constraint convex in what follows.
\par
\paragraph*{Kernel PCA from SVD:}
The usual way to solve the KPCA problem is to express the directions $W = [w_j]_{j=1}^s$ as linear combinations of the features, introducing coefficients $(\alpha_{ij})_{i,j=1}^{n,s} \in \R^{n \times s}$ such that 
\begin{equation*}
  w_j = \sum_{i=1}^n \alpha_{ij} \phi(x_i).
\end{equation*}
Quick algebraic manipulations then ensure that the solution to \Cref{pbm:primal_kpca} can be obtained by taking the coefficients $\alpha$ to be the top-$s$ eigenvectors of the Gram matrix $G = [k(x_i, x_j)]_{i,j=1}^n$, rescaled so that the directions have unit norm.
Thus KPCA is solved by performing the SVD of some $n \times n$ matrix, an operation scaling as $\mO(n^3)$ and quite slow even for moderate sizes of datasets.
Throughout this paper we make the assumption that $G$ is full rank, which typically happens when the feature space is infinite dimensional (\eg Gaussian kernel) and the data does not contain any duplicate.
This assumption is critical to the derivation of the gradient in \Cref{sec:optimization}.
\paragraph{Difference of convex functions:}
The KPCA optimization problem can be cast into the minimization of a difference of convex functions of the form
\begin{problem}
  \label{pbm:primal_dc}
  \inf_{W \in \mH^s} ~ g(W) - f(\Gamma W).
\end{problem}
Indeed, plugging back $f = \frac{1}{2} \norm{\cdot}_{\fro}^2$ and $g = \iota_{\mS_\mH^s}(\cdot)$ into \Cref{pbm:primal_dc} one gets back the original formulation from \Cref{pbm:primal_kpca}.
One of the advantages of this formulation is that it becomes possible to slightly modify the loss function $f$ so as to enforce specific properties such as robustness or sparsity for the solution.
\par
The goal of this work is two-fold:
\begin{itemize}
  \item Exploit gradient-based optimization schemes to solve KPCA without using the SVD of $G$, resulting in computational advantages (see \Cref{sec:optimization}).
  \item Extend the KPCA problem to a larger set of objective functions, with the benefit of enforcing robust or sparse solutions (see \Cref{sec:flexible}).
\end{itemize}
\section{Solving Kernel PCA with Gradient Descent}
\label{sec:optimization}
In this section, we use duality principles to derive suitable optimization algorithms to solve the KPCA problem introduced in \Cref{pbm:primal_kpca}.
We begin by a general proposition performing the dualization of problems that can be written as the minimization of a difference of two convex functions.
The proof is similar to \citet{toland1979subdifferential}, with an additional care given to handling the operator $\Gamma$.
\begin{proposition}[Dual of difference of convex functions]
\label{prop:general_dual}
Let $\cU, \cK$ be two Hilbert spaces, $g \colon \cU \to \bar{\R}$ and $f \colon \cK \to \bar{\R}$ be two convex lower semi-continuous functions and $\Gamma \in \cL(\cU, \cK)$.
The problem
\begin{equation*}
    \inf_{W \in \cU} ~ g(W) - f(\Gamma W)
\end{equation*}
admits the dual formulation
\begin{problem}
	\label{pbm:dual_dc}
    \inf_{H \in \cK} ~ f^\star(H) - g^\star(\Gamma^\sharp H),
\end{problem}
and strong duality holds.
\end{proposition}
The main motivation for going from the primal problem to the dual in the KPCA case is that the dual variable $H$ is a finite dimensional matrix, suitable to gradient-based optimization schemes.
However this comes at the cost of having to handle the term $g^\star(\Gamma^\sharp H)$ which encodes information related to the Stiefel manifold.
In \Cref{sec:nuclear_norm}, we show that this term is related to the nuclear norm of some low dimensional matrix and derive a gradient for it, before exploring in \Cref{sec:critical_dual} the link between the critical points of the dual function and those of the reconstruction cost associated to the Gram matrix.
Finally in \Cref{sec:optimal_dual}, we show how to compute the projections associated to new points without having to compute the primal solution.
\subsection{Nuclear Norm Gradient Computation}
\label{sec:nuclear_norm}
The following proposition instantiates the term $g^\star(\Gamma^\sharp H)$ for the KPCA problem.
\begin{proposition}
	\label{prop:fl_gstar}
	Let $g$ be the indicator function of the Stiefel manifold and $\Gamma$ as in \Cref{pbm:primal_kpca}.
	Then for all $H \in \R^{n \times s}$,
	\begin{equation}
		\label{eq:pi}
		g^\star(\Gamma ^\sharp H) = \Tr \sqrt{H^\top GH}
	\end{equation}
\end{proposition}
We recognize in \Cref{eq:pi} the nuclear norm of the matrix $\sqrt{H^\top GH}$, which is well-defined since $G$ is a Gram matrix associated to a positive definite kernel.
\begin{remark}
	In the proof, we first show that $g^\star(\Gamma ^\sharp H) = \norm{G^{\frac{1}{2}}H}_{S_1}$ where $\norm{\cdot}_{S_1}$ is the Schatten $1$-norm, \ie the nuclear norm.
	While the dependency in $H$ makes this easier to handle from an optimization standpoint, the dependency in $G^{\frac{1}{2}}$ would require to perform the SVD that we want to avoid, leaving us to work with $\Tr \sqrt{H^\top GH}$.
\end{remark}
In what follows, we define $\pi(H) := \Tr \sqrt{H^\top GH}$ and build on \citet{lewis1996derivatives} to give a gradient expression for $\pi$ used in subsequent optimization algorithms.
\begin{proposition}
	\label{prop:gradient_spectral_norm}
	If all eigenvalues of $H^\top G H$ are positive then $\pi$ is differentiable at $H$ with gradient
	\begin{equation}
		\label{eq:gradient_spectral_norm}
		\nabla \pi(H) = GHU^\top \mathrm{diag}\left( \frac{1}{\sqrt{\lambda(H^\top GH)}} \right)U,
	\end{equation}
	where $U \in \R^{s \times s}$ is an orthogonal matrix satisfying $H^\top G H = U^\top \mathrm{diag}\left( \lambda(H^\top GH) \right) U$.
\end{proposition}
To compute the gradient given in \Cref{eq:gradient_spectral_norm}, one needs to perform the SVD of the matrix $H^\top GH \in \reals^{s \times s}$.
In the context of a (relatively) small number of components, this SVD is computationally cheap.
This motivates the use of gradient-based method to obtain the dual solution faster than by exploiting the SVD of the $n \times n$ matrix $G$.

More precisely, the computational complexity associated to the computation of this gradient can be bounded by the sum of 
\begin{itemize}
	\item Computation of $H^\top GH$ in $\mathcal{O}(sn^2)$
	\item SVD of $H^\top GH$ in $\mathcal{O}(s^3)$
	\item Computation of \Cref{eq:gradient_spectral_norm} in $\mathcal{O}(ns^2 + s^3)$ by reusing the precomputed $GH$.
\end{itemize}
\begin{remark}
	Due to the dependency of the gradient in $\frac{1}{\sqrt{\lambda({H^\top GH})}}$, it is not Lipschitz continuous in $H$,
which prevents the use of fixed stepsizes schemes in classical optimization algorithms.
\end{remark}
\subsection{Critical Points of the Dual Problem}
\label{sec:critical_dual}
Overall, the dual problem to KPCA reads
\begin{problem}
	\label{pbm:dual_kpca}
	\inf_{H \in \R^{n \times s}} \frac{1}{2} \Tr(H^\top H)  - \pi(H).
\end{problem}
Solving this nonconvex problem can be challenging, as local minima or saddle points phenomena can occur.
Typically, algorithms such as the DC algorithm \citep{tao1997} ensure convergence towards a critical point, that is a point $H$ at which the gradient is nullified.
In the following proposition, we explicit the link between critical points of the dual function and those of the reconstruction cost associated to the Gram matrix.
\begin{proposition} \label{prop:link_critical_reconstruction}
	Let $J(H) = \frac{1}{4} \norm{G - H H^\top }^2_\fro$.
	Let $\hat{H}$ be a critical point of \Cref{pbm:dual_kpca}.
	Then $\nabla J (\hat{H}) = 0$. 
\end{proposition}
The characterization of the critical points of $J$ is given in \citep{Wright-Ma-2022}.
We solve the dual problem \eqref{pbm:dual_kpca} using the L-BFGS optimization algorithm, which in practice we observe always avoids sub-optimal critical points.
More details on the optimization algorithm are given in \Cref{sec:experiments}.
\subsection{Exploiting the Dual Solution}
\label{sec:optimal_dual}
Assume that solving \Cref{pbm:dual_kpca} has produced an optimal dual solution that we denote $\hat{H}$.
In general, finding the corresponding directions in the feature space (primal variables) involves solving the optimization problem 
\begin{problem}
	\hat{W} \in \argmin_{W \in \mH^s} ~~ g(W) - \langle \hat{H}, \Gamma W \rangle.
\end{problem}
This optimization problem is non-trivial when we do not have more information about the solution $\hat{H}$.
However, when $\hat{H}$ is obtained by the SVD of $G$, there is a linear dependency between the directions in the feature space and the dual variables.
This is exemplified in the following proposition where we show that the SVD of $G$ allows to pick an optimal dual solution.
\begin{proposition}
	\label{prop:optimal_svd_dual}
	Let $G = U^\top \Sigma U$ be the SVD of $G$ with $U \in \R^{n \times n}$ being orthonormal and $\Sigma$ being diagonal.
	Let $H^{\text{svd}} = U_s^\top \sqrt{\Sigma}_s$ where $U_s \in \R^{s \times n}$ gathers the top-$s$ eigenvectors of $G$ and $\Sigma_s \in \R^{s \times s}$ the top-$s$ eigenvalues.
	Then $H^{\text{svd}}$ is a solution to \Cref{pbm:dual_kpca}.
\end{proposition}
In our case, we do not have access to such an optimal solution and must find a way to compute the projections with the only knowledge of $\hat{H}$.
It turns out that recovering the projection on the principal components is possible using the kernel trick, as proposed in the following.
\begin{proposition}
	\label{proposition:projections}
	Let $x \in \mX$ and $G_x = [k(x, x_i)]_{i=1}^n \in \R^{n}$.
	Let $\hat{W} \in \mH^s$ be a solution to \Cref{pbm:primal_kpca} and $\hat{H} \in \R^{n \times s}$ be a solution to \Cref{pbm:dual_kpca}.
	The projections of $\phi(x)$ onto the principal components are given by 
	\begin{equation}
		\label{eq:projections}
		[\langle \phi(x), \hat{w}_j \rangle]_{j=1}^s = G_x ^\top \hat{H} U^\top \mathrm{diag}\left( \lambda(\hat{H}^\top G \hat{H})\right)^{-\frac{1}{2}} U,
	\end{equation}
	where $U$ is obtained from the SVD of $\hat{H}^\top G \hat{H}$ as in \Cref{prop:gradient_spectral_norm}.
\end{proposition}
Note that the expression in \Cref{eq:projections} does not directly involves the directions in the feature space encoded in $\hat{W}$.

\section{Beyond Variance Maximization}
\label{sec:flexible}
In this section, we propose to modify the variance objective used in the original KPCA problem to promote desirable properties such as sparsity or robustness in the dual variable $H$.
The idea is based on using objectives obtained from infimal convolution with the squared norm, known as \emph{Moreau envelopes} \citep{moreau1965proximite}.
We focus on variance-like objectives of the form 
\begin{equation*}
    f = \frac{1}{2} \norm{\cdot}^2_\fro \infconv \Psi,
\end{equation*}
where $\Psi \colon \R^{n \times s} \to \bar{\R}$ is a well-chosen function enforcing desirable properties.
Compatibility between the Fenchel-Legendre transform and the infimal convolution operator then allows to write the dual to \Cref{pbm:primal_dc} as
\begin{problem}
    \label{pbm:dual_dc_infconv}
    \inf_{H \in \R^{n \times s}} ~ \frac{1}{2} \norm{H}^2_\fro + \Psi^\star(H) - \pi(H).
\end{problem}
The section is organized as follows: we begin in \Cref{sec:huber_epsilon} by expliciting some choices of $\Psi$ that give rise to the Huber and $\epsilon$-insensitive objectives (known to respectively promote robustness and sparsity) before exploiting their Moreau envelope structure to design dedicated optimization schemes in \Cref{sec:dca}.
\subsection{Huber and $\epsilon$-insensitive Objectives}
\label{sec:huber_epsilon}
In what follows, let $\norm{\cdot}$ be a norm on $\R^{n \times s}$ and $\norm{\cdot}_\star$ be its dual norm.
Given $t \geq 0$, the balls of radius $t$ for these norms are respectively denoted as $\cB_t$ and $\cB^\star_t$.
\begin{definition}[Huber]
    Let $\kappa > 0$. The Huber objective with parameter $\kappa$ is defined as
    \begin{equation*}
        H_\kappa := \frac{1}{2} \norm{\cdot}^2_\fro \infconv \kappa \norm{\cdot}.
    \end{equation*}
\end{definition}
The Huber objective can be understood as the Moreau envelope of $\kappa \norm{\cdot}$ with parameter $1$.
This case corresponds to $\Psi = \kappa \norm{\cdot}$ and $\Psi^\star = \iota_{\cB^\star_\kappa}$.
The additional term in \Cref{pbm:dual_dc_infconv} constrains the dual variable $H$ to pertain to the ball of radius $\kappa$ for the norm $\norm{\cdot}_\star$, thus inducing robustess.
While the choice $\norm{\cdot} = \norm{\cdot}_\fro$ recovers the classical version of the Huber objective, it is to be noted that varying choices of norms allow to capture different notions of robustness \citep{lambert2022}.
\begin{definition}[$\epsilon$-insensitive]
    Let $\epsilon > 0$. The $\epsilon$-insensitive objective with parameter $\epsilon$ is defined as
    \begin{equation*}
        \ell_\epsilon := \frac{1}{2} \norm{\cdot}^2_\fro \infconv \iota_{\cB_\epsilon}.
    \end{equation*}
\end{definition}
The $\epsilon$-insensitive objective is the Moreau envelope of some indicator function of a ball, corresponding to $\Psi = \iota_{\cB_\epsilon}$ and $\Psi^\star = \epsilon \norm{\cdot}_\star$.
Depending on the choice of the norm, the additional term in \Cref{pbm:dual_dc_infconv} can behave like a Lasso or group Lasso penalty, inducing sparsity in the iterates.
\subsection{Solving with DC Algorithm}
\label{sec:dca}
We propose to solve \Cref{pbm:dual_dc_infconv} using the well-known difference of convex functions algorithm (DCA) \citep{tao1997}.
The point of DCA is to search for a critical point of the problem, that is some $H \in \R^{n \times s}$ such that 
\begin{equation*}
    \partial \pi(H) \cap \partial (\frac{1}{2}\norm{\cdot}_\fro^2 + \Psi^\star)(H) \neq \emptyset.
\end{equation*}

\begin{table}[t]
	\centering
	\caption{Proximal operators for Huber and $\epsilon$-insensitive objectives.}
	\label{tab:proximal_psi}
	\begin{tabular}{lccc}
		\toprule
		Objective & $\Psi$ & $\Psi^\star$ & $\prox_{\Psi^\star}(Y)$ \\
		\midrule
		$H_\kappa$ & $\kappa \norm{\cdot} $ & $\iota_{\cB^{\star}_\kappa}$  & $\Proj_{\cB^{\star}_\kappa}(Y)$  \\
		$\ell_\epsilon$ & $\iota_{\cB_\epsilon}$ & $ \epsilon \norm{\cdot}_\star $ & $Y - \Proj_{\cB_\epsilon}(Y) $ \\
		\bottomrule
	\end{tabular}
\end{table}

The search is performed sequentially in the subgradient of $\pi$, and in the subgradient of the Fenchel-Legendre transform of $\frac{1}{2}\norm{\cdot}_\fro^2 + \Psi^\star$, which is $f$.
As $f$ was chosen to be a Moreau envelope, its gradient is always defined as for all $Y \in \R^{n \times s}$,
\begin{equation*}
    \nabla \left (\frac{1}{2} \norm{\cdot}^2_\fro \infconv \Psi \right )(Y) = Y - \prox_\Psi(Y).
\end{equation*}
According to Moreau decomposition \citep{moreau1962fonctions}, it moreover holds that for all $Y \in \R^{n \times s}$,
\begin{equation*}
    Y - \prox_\Psi(Y) = \prox_{\Psi^\star}(Y).
\end{equation*}
Thus the DCA algorithm can be applied as long as the computation of $\prox_{\Psi^\star}$ is possible.
The correspondence between the choice of objective and the resulting proximal operator is gathered in \Cref{tab:proximal_psi}.
In particular, one can note that tractability of the DCA algorithm requires that the projection on the balls for the dual norm $\norm{\cdot}_\star$ is possible in the Huber case, whereas for $\epsilon$-insensitive objectives what matters is being able to project on the balls for the norm $\norm{\cdot}$.
This requirement drives the choice of the norm used in the definition of the objectives, as projecting on balls can be obtained in closed-form when the norm is a $2$-norm or an $\infty$-norm.
The DCA is summarized in \Cref{alg:dca_moreau}.
\LinesNumberedHidden{
\begin{algorithm}[t]
\SetKwInOut{Input}{input}
\SetKwInOut{Init}{init}
\SetKwInOut{Parameter}{Param}
\SetInd{0.5em}{0.5em}
\caption{DCA for Moreau envelope objectives 
\label{alg:dca_moreau}}
\Input{~Gram matrix $G$}
\Init{$H^{(0)} \in \mathbb{R}^{n \times s}, Y = 0$}
\For{epoch $t$ from $0$ to $T-1$}{
    \tcp{alternating gradient steps}
    $Y = \nabla \pi (H^{(t)})$ from \Cref{eq:gradient_spectral_norm}

    $H^{(t+1)} = \prox_{\psi^\star}(Y)$
    }
\Return{$H^{(T)}$}
\end{algorithm}}

\section{Numerical Experiments}
\label{sec:experiments}

Through numerical evaluations, we show the efficiency and flexibility of the proposed KPCA dualization on a diverse collection of 
dimensionality reduction tasks. 

First, we evaluate the efficiency of the resulting LBFGS-based optimization algorithm for KPCA on multiple datasets, with comparisons to other standard KPCA solvers. 
We apply the proposed algorithm to problems of different sizes, and we also study the effect on performance of the decay of the eigenspectrum of $G$.

Then, we quantify the robustness of the Huber losses on data contaminated with outliers depending on the noise level, and we study the influence of the loss parameters. 
Regarding the $\epsilon$-insensitive losses, we study the accuracy-sparsity tradeoff w.r.t. the $\epsilon$ parameter and the number of components $s$.

Experiments are implemented in Python 3.10 on a machine with a 3.7GHz Intel i7-8700K processor and 64GB RAM.
The  code is available  at \url{https://github.com/taralloc/dc-kpca}.

\begin{table*}[!htbp]
	\caption{\textbf{KPCA Training Time}. 	
	Runtime for multiple KPCA problems with \underline{higher} tolerance. 
	Speedup factor w.r.t. RSVD. 
	}
	\label{tab:speedup:high}
	\centering
	\begin{tabular}{lcccccc}
		\toprule
		\multirow{2}{*}{Task} & \multirow{2}{*}{$n$} & \multicolumn{4}{c}{Time (s) for $\delta=10^{-2}$}  & \multicolumn{1}{c}{Speedup} \\\cmidrule(lr){3-6}
		& & SVD & Lanczos & RSVD & Ours & Factor\\	
		\midrule
		Synth 1     & 7000  & 96.73  &  0.85  & 1.97   & \textbf{0.53}  & 3.72 \\ 	
		Protein     & 14895 & 868.64 &  3.46  & 6.70   & \textbf{1.07}  & 6.25 \\ 	
		RCV1        & 20242 & -      &  6.04  & 12.50  & \textbf{2.12}  & 5.90 \\ 
		CIFAR-10    & 60000 & -      &  48.10 & 123.89 & \textbf{13.51} & 9.17 \\ 	
		\bottomrule
	\end{tabular}
\end{table*}

\begin{table*}[t]
	\caption{\textbf{KPCA Training Time}. 
	Runtime for multiple KPCA problems with \underline{lower} tolerance.
	}
	\label{tab:speedup:low}
	\centering
	\begin{tabular}{lcccccc}
		\toprule
		\multirow{2}{*}{Task} & \multirow{2}{*}{$n$} & \multicolumn{4}{c}{Time (s) for $\delta = 10^{-4}$}  & \multicolumn{1}{c}{Speedup} \\\cmidrule(lr){3-6}
		& & SVD & Lanczos & RSVD & Ours & Factor\\	
		\midrule
		Synth 1     & 7000  & 96.73  &   4.78  &  1.97  & \textbf{0.56}  & 3.52 \\ 	
		Protein     & 14895 & 868.64 &   3.51  &  6.95  & \textbf{1.07}  & 6.59 \\ 	
		RCV1        & 20242 & -      &  19.72  & 12.75  & \textbf{3.78}  & 3.38 \\ 
		CIFAR-10    & 60000 & -      &  48.15  & 122.58 & \textbf{29.26} & 4.19 \\ 	
		\bottomrule
	\end{tabular}
\end{table*}

\paragraph{Datasets.}%
We evaluate our approach on synthetic and real-world datasets. 
Synth 1 ($n=7000, d=10000$) is a high-dimensional synthetic dataset where samples are drawn randomly from a multivariate normal distribution with zero mean and fixed covariance matrix. 
For real-world data, we consider datasets from LIBSVM \citep{libsvm}, UCI \citep{uci}, and common deep learning benchmarks: Iris ($n=150, d=4$), the bioinformatics dataset Protein ($n=14895, d=357$), the text categorization dataset RCV1 ($n=20242, d=47236$), and the outputs of the second to last layer of a ResNet18 \citep{he2016} trained on the computer vision dataset CIFAR-10 ($n=60000, d=512$).

\paragraph{Experimental setups.}%
Regarding optimization for KPCA with the square loss, we employ the LBFGS algorithm with backtracking linesearch using the strong Wolfe conditions with initialization from the standard normal distribution. 
For the Moreau envelopes, i.e. Huber and $\epsilon$-insensitive losses, we employ the DC algorithm; for faster convergence second order optimization algorithms for this problem with composite smooth + non-smooth structure can also be employed, e.g. \citep{stella2017}.
For the square loss, the accuracy of a dual iterate $H_k$ from our solver at each iteration $k$ is chosen as the relative difference $\eta=|d(H_k)-d_\text{opt}| / d_\text{opt}$ between the dual cost $d(H_k)=\tfrac{1}{2}\trace{H_k^\top H_k}-\trace{\sqrt{H_k^\top G H_k}}$ at iteration $k$ and the optimal dual cost $d_\text{opt}=-\tfrac{1}{2}\sum_{i=1}^s \lambda_i$, with $\lambda_i$ being the $i$-th largest eigenvalue of $G$. In fact, the optimal primal cost is the highest variance in $s$ components, i.e., $\tfrac{1}{2}\sum_{i=1}^s \lambda_i$, and strong duality holds in our dualization.
For all used solvers, we use the same stopping criterion based on achieving a target tolerance. 
For the other convoluted losses, we stop the DCA when the absolute variation of the loss is less than machine precision with at most 1000 iterations.

\subsection{More Efficient KPCA}

\begin{figure}[t]
	\centering
	\includegraphics{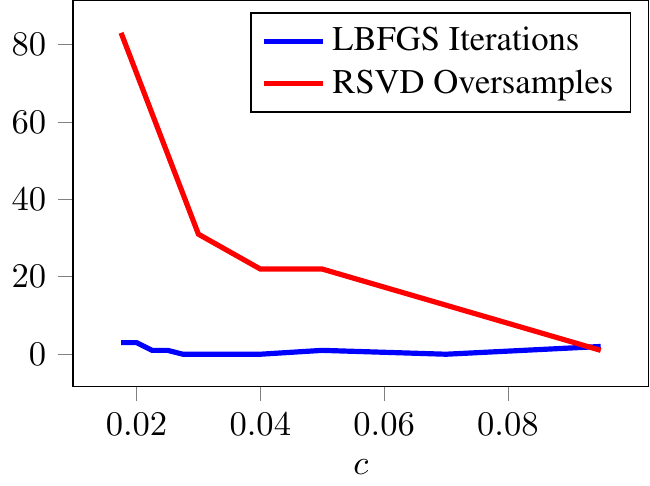}
	\caption{\textbf{Varying eigenspectrum.}
	Additional computational burden when the spectrum of $G$ changes (larger $c$ corresponds to spectra with faster decay).
	Blue: our LBGS-based algorithm, red: randomized SVD.
	}
	\label{fig:spectrum}
\end{figure}

This subsection demonstrates the performance of the proposed LBFGS-based algorithm to solve the KPCA problem \eqref{pbm:primal_kpca}. 
We compare our method with three common KPCA solvers: full SVD (SVD), Implicitly Restarted Lanczos Method \citep{lehoucq1998}, and randomized SVD (RSVD) \citep{halko2011}. 
The Lanczos solver finds the first $s$ eigenvalues and corresponding eigenvectors of the symmetric matrix $G$, while randomized SVD finds the truncated singular value decomposition of $G$ by random projections.
These solvers rely on different stopping criteria. To make a fair comparison, note that $d_\text{opt}=d(U\sqrt{S})$, with $G=USV^\top$ being the SVD of $G$. Therefore, the accuracy of an approximate solution $\hat{U}, \hat{S}, \hat{V}$ is measured by the following relative dual cost residual:
\begin{equation}
	\eta = \abs{d(\hat{U}\sqrt{\hat{S}}) - d_\text{opt}} / d_\text{opt},
\end{equation}
where for RSVD $\hat{S}$ is the diagonal matrix of the largest $s$ singular values of $G$ and corresponding computed singular vectors $\hat{U}, \hat{V}$, while for the eigendecomposition solver $\hat{S}$ is the diagonal matrix of the largest $s$ eigenvalues of $G$ and corresponding eigenvectors $\hat{U}=\hat{V}$ found by the Lanczos solver.
Full SVD is run to machine precision for comparison.
For a given tolerance $\delta$, we stop training when $\eta < \delta$. 
In particular, for RSVD, the number of required oversamples is found by increasing the number of oversamples until the target tolerance is reached.

The kernel is chosen to be the Laplace kernel $k(z, y)=\exp\left(-\norm{z-y}_2/(2\sigma^2)\right)$ with $\sigma=0.1\sqrt{d\sigma_x}$ and $\sigma_x$ the variance of the training data. 
For the KPCA dimensionality reduction task, it is common to assume that the given high-dimensional data can be expressed over a small number of principal components; therefore, in these experiments, we use $s=20$. 
The experiments are averaged over 5 runs.

Table \ref{tab:speedup:high} and \ref{tab:speedup:low} show the training times on different KPCA tasks for multiple tolerance levels $\delta=10^{-2}, 10^{-4}$; lowest training times are in bold. The time to compute the full SVD to machine precision is given for reference in the SVD column, where ``-" indicates that SVD took longer than 30 minutes.
The speedup factor is $t^\text{(RSVD)}/t^\text{(LBFGS)}$, where $t^\text{(RSVD)}, t^\text{(LBFGS)}$ is the training time using the RSVD solver and our LBFGS-based solver, respectively.
For tolerance $\delta=10^{-2}$, our solver is faster than all other KPCA solvers and at least 3 times faster than RSVD. 
For the lowest tolerance $\delta=10^{-4}$, our solver is the fastest with generally smaller speedup factor. 
Since our solver is second order, it can reach higher accuracy quickly, with subsequent iterations giving relatively lower improvement, which explains why the speedup is more apparent at higher tolerances.
All solvers can achieve high accuracy; 
however, in some tasks (e.g., CIFAR-10 and Protein) the Lanczos method cannot take significant advantage of higher tolerance requirements.

A solver's performance depends on the properties of the given data and kernel. 
For instance, it is known that RSVD requires more oversamples for matrices whose eigenspectrum decays slowly \citep{halko2011}, which is often the case in real-world problems. 
On the other hand, our LBGS-based solver does not suffer from such issue. 
To further illustrate this point, we construct $G$ from random data $X$ as $G = 0.01 (X+X^\top) + U D U^\top$, where $U$ is any orthogonal matrix and $D$ is a diagonal matrix s.t. $D_{ii} = \exp(-ci)$.
Varying $c$ controls how quickly the spectrum of $G$ decays, with lower $c$ corresponding to slower decays. 
In Fig. \ref{fig:spectrum}, we vary $c$ and show the additional computational cost for RSVD and for our solver when the spectrum of $G$ changes.
For RSVD, the red line shows the additional required oversamples needed to achieve tolerance $\delta=10^{-4}$ w.r.t. the experiment with lowest number of oversamples. 
The blue line shows the additional LBFGS iterations required by our solver to achieve the same tolerance. 
While RSVD requires significantly more oversamples to reach a fixed accuracy as $c$ decreases, our solver takes a similar number of iterations, showing that our solver is mostly unaffected by this type of change in the eigenspectrum of $G$.

The influence of $s$ on training time is studied in \Cref{fig:r} on random data with $n=15000$: higher $s$ leads to longer training times for both LBFGS and RSVD, with LBFGS maintaining performance advantage. Very high $s$ impacts LBFGS's training time as it needs to compute the SVD of a $s \times s$ matrix at every iteration.

\begin{figure}[!htbp]
	\centering
	\includegraphics{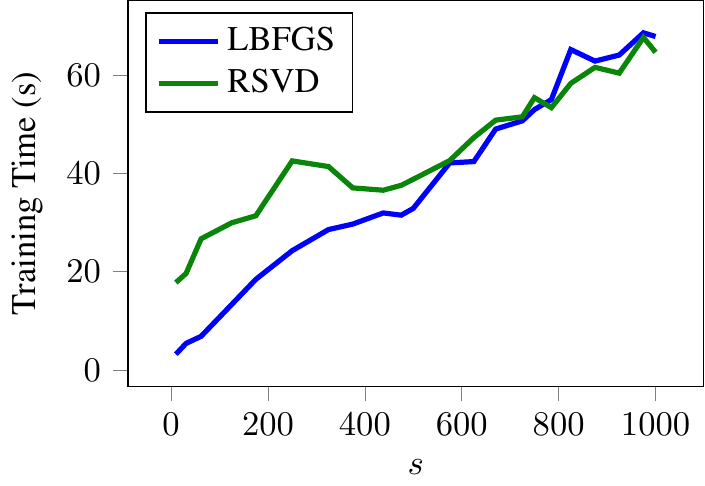}
	\caption{\textbf{KPCA Training Time}. Effect of $s$.}
	\label{fig:r}
\end{figure}
	
\begin{table}[!htbp]
	\centering
	\caption{\textbf{Robustness}. MSE on contaminated Iris dataset.}
	\label{tab:huber}
	\begin{tabular}{lccc}
		\toprule
		$\tau$ & $\tfrac{1}{2}\norm{\cdot}^2$ & $H_\kappa^2$ & $H_\kappa^1$ \\
		\midrule
		10 & 7.591059 & 6.833484 & 7.381284 \\
		25 & 7.910846 & 7.182663 & 7.687518 \\
		50 & 8.691805 & 8.045477 & 8.430957 \\
		75 & 9.782740 & 9.259353 & 9.465706 \\
		100 & 11.183650 & 10.824293 & 10.791766 \\
		\bottomrule
	\end{tabular}
\end{table}

\subsection{Huber Losses} \label{sec:experiments:huber}

We now present numerical experiments to illustrate the flexibility of the KPCA dualization framework applied to the Huber loss. 
We consider the Huber losses associated to different norms: $H_\kappa^1$ is the Moreau envelope of $H \mapsto \kappa \norm{H}_1$ and $H_\kappa^2$ is the Moreau envelope of $H \mapsto \max_{i \in [n]} \norm{h_i}_2$.
Both cases are optimized with the same DC algorithm.

\begin{figure}
	\centering
	\includegraphics{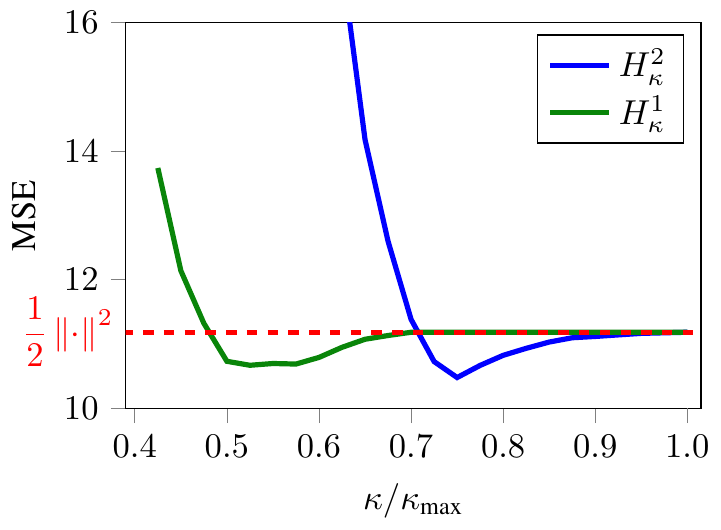}
	\caption{\textbf{Robustness}. Effect of $\kappa$ for the loss $H_\kappa^2$ and $H_\kappa^1$. %
	}
	\label{fig:msekappa}
\end{figure}

\begin{figure*}[t]
	\centering
	\begin{subfigure}[b]{0.48\textwidth}
		\raggedleft
		\includegraphics{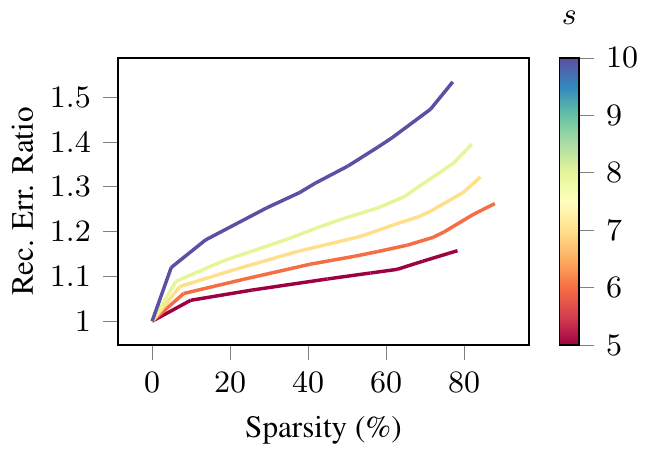}
	\end{subfigure}%
	\hspace{5pt}
	\begin{subfigure}[b]{0.48\textwidth}
		\raggedright
		\includegraphics{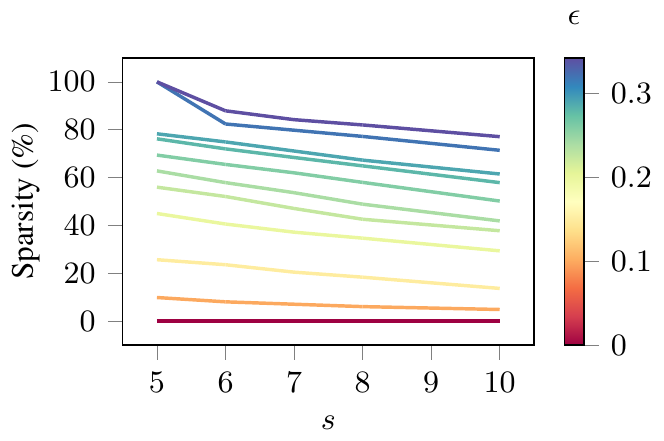}
	\end{subfigure}
	\caption{\textbf{Sparsity}. Reconstruction error for the $\ell_\epsilon^2$ loss for multiple $\epsilon$ and $s$.}
    \label{fig:sparsity:block}
\end{figure*}

\begin{figure*}[t]
	\centering
	\begin{subfigure}[b]{0.48\textwidth}
		\raggedleft
		\includegraphics{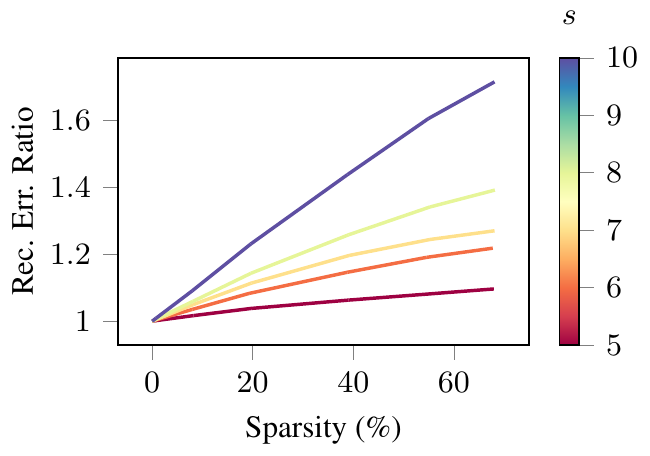}
		\label{fig:sparsity:coordinate:1}
	\end{subfigure}%
	\hspace{5pt}
	\begin{subfigure}[b]{0.48\textwidth}
		\raggedright
		\includegraphics{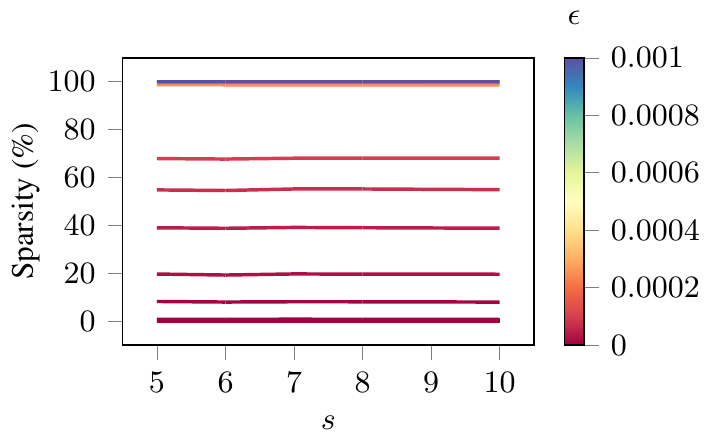}
		\label{fig:sparsity:coordinate:2}
	\end{subfigure}
	\caption{\textbf{Sparsity}. Reconstruction error for the $\ell_\epsilon^\infty$ loss for multiple $\epsilon$ and $s$.}
	\label{fig:sparsity:coordinate}
\end{figure*}

We first investigate the robustness induced by the Huber losses on the Iris dataset. 
We contaminate the data in the following way: to corrupt data $\{x_i\}_{i=1}^n$, we first draw a set $I \subset \{0, ..., n\}$ s.t. $|I| = \lfloor \omega n \rfloor$, with $\omega \in [0,1]$ being the proportion of corrupted samples. 
Then, we introduce outliers using multiplicative Gaussian noise with zero mean and $\tau$ standard deviation: for $i \in I$, $x_i$ is replaced by $b_i x_i$, where $b_i$ is drawn from $\mathcal{N}(0,\,\tau^2)$. 
We evaluate performance by reconstruction error (MSE) in input space of non-contaminated test samples: higher MSE means that the learned subspace is influenced by the outliers, while lower MSE corresponds to more robust models. 
We compute pre-images using the technique of \citep{bakir2004}.
As discussed in Section \ref{sec:optimization}, as $\kappa$ grows, the constraint on the dual variables becomes void and we recover the square loss KPCA problem. Therefore, we set $\kappa=0.6\kappa_\text{max},\,0.8\kappa_\text{max}$ for $H_\kappa^1,\,H_\kappa^2$, respectively.

The resulting test MSE values are given in Table \ref{tab:huber} for $\omega = 8\%$ on a 20\% test split. 
We employ a Gaussian kernel $k(z, y)=\exp\left(-\norm{z-y}_2^2/2\right)$ for all models. 
The Huber losses $H_\kappa^1,\,H_\kappa^2$ are more robust to outliers for all noise levels. 
Particularly, $H_\kappa^1$ show greater robustness at the highest noise level, while lower noise levels favor $H_\kappa^2$.
Overall, these results show that the learned principal components are more robust to the introduced outliers.

The influence of $\kappa$ is studied in Fig. \ref{fig:msekappa} with $\tau=100$. The Huber losses (blue and green lines) can be distinctively more robust to outliers than the square loss (red dashed line). When $\kappa$ becomes closer to $\kappa_\text{max}$, the robustness effect is void and the MSE converges to the one from the standard KPCA. With small $\kappa$, the constraint on the dual variables is too strict to learn meaningful principal components, as the size of the projection ball becomes insufficient. Accordingly, the MSE grows quickly in the small $\kappa$ region. A balanced choice of $\kappa$ results in lower MSE, i.e., reduces the influence of the outliers and therefore learns more robust principal components.

\subsection{$\epsilon$-insensitive Losses} \label{sec:experiments:epsilon}

The $\epsilon$-insensitive losses induce sparsity on $H$.
The choice of the ball used to define the Moreau envelope determines the sparsity type through the dual norm and $\epsilon$ affects the sparsity level, where $\epsilon=0$ recovers the square loss case.
In particular, for the dual norms $\norm{H}_\star = \norm{H}_1$ and $\norm{H}_\star = \sum_{i=1}^n \norm{h_i}_2$, we obtain respectively the losses $\ell_\epsilon^\infty,\,\ell_\epsilon^2$. The former promotes unstructured sparsity as the associated proximal step involves coordinate-wise soft-thresholding, the latter promotes block sparsity as the associated proximal step involves block soft-thresholding. The use of convoluted losses makes it possible to cast both the robustness and sparsity losses in the same duality framework. Accordingly, as for the Huber loss family, we address this optimization problem through the DC algorithm.

The study of the role of $\epsilon$ and $s$ for the $\ell_\epsilon^2$ loss is shown in Fig. \ref{fig:sparsity:block}, where we consider a random data matrix $X$ with $n=1000,\,d=20$ and the kernel is chosen to be Gaussian.
The reconstruction error ratio is the ratio between the reconstruction error using the components learned with the $\ell_\epsilon^2$ loss and the reconstruction error from the square loss. 
Here, sparsity is in terms of percentage of zero rows in $H$. 
In fact, typically the dual solution from KPCA is dense and all training points, corresponding to the rows of $H$, contribute. 
Block sparsity therefore induces a representation in a fraction of the training points.
We can see that setting $\epsilon=0$ recovers the square loss case. 
As $\epsilon$ grows, $H$ becomes sparser. 
We also expect that, as $\epsilon$ grows, the performance declines in terms of reconstruction error, demonstrating a trade-off between sparsity and accuracy of the learned representation.
Such compromise depends on the requirements of specific applications.
In the example of Fig. \ref{fig:sparsity:block}, with $s=5$ principal components, one can obtain 40\% sparsity with just a 10\% increase in reconstruction error.
Increasing $s$ with fixed $\epsilon$ usually leads to a slight decrease in sparsity, meaning that a higher $\epsilon$ is needed to achieve a similar sparsity level.
The $\ell_\epsilon^2$  performance relative to the square loss declines with higher $s$ for a fixed sparsity level. Similar conclusions can be drawn for the $\ell_\epsilon^\infty$ loss from Fig. \ref{fig:sparsity:coordinate}, where sparsity is in terms of zero entries of $H$.

\section{Conclusion}
\label{sec:conclusion}

This work presents a duality framework for the kernel PCA problem seen as a difference of convex functions.
The generalized family of objectives with Moreau envelopes structure allows to extend the variance maximization problem to a wider choice of objectives, inducing robust and sparse estimators.
The resulting gradient-based algorithm for standard KPCA shows important speedups in training time compared to the SVD solvers.
Future work could focus on convergence properties of the DC algorithm towards optimal critical points.
\section*{Acknowledgements}
\label{sec:ack}
This work is jointly supported by ERC Advanced Grant E-DUALITY (787960), iBOF project Tensor Tools for Taming the Curse (3E221427), Research Council KU Leuven: Optimization
framework for deep kernel machines C14/18/068, KU Leuven Grant CoE PFV/10/002, and Grant  FWO G0A4917N, EU H2020 ICT-48 Network TAILOR (Foundations of Trustworthy AI - Integrating Reasoning, Learning and Optimization), and the Flemish Government (AI Research Program), and Leuven.AI Institute. This work was also supported by the Research Foundation Flanders (FWO) research projects G086518N, G086318N, and G0A0920N; Fonds de la Recherche Scientifique — FNRS and the Fonds Wetenschappelijk Onderzoek — Vlaanderen under EOS Project No. 30468160 (SeLMA). 

\bibliography{references}
\bibliographystyle{icml2023}
\newpage
\onecolumn
\appendix
\section{Proofs}
\subsection{Proof of \texorpdfstring{\Cref{prop:general_dual}}{Proposition 3.1}}
\begin{proof}
    Since $f$ is proper closed convex, it holds that $f=f^{**}$, so equivalently we can write \Cref{pbm:primal_dc} as
    \begin{align}
        \inf_{W\in\cU} g(W) - \sup_{V\in\mathcal{K}} \, \{ \langle V, \Gamma W \rangle - f^*(V) \}
        \label{eq:dual:start}
        &=
        \inf_{W\in\cU} g(W) + \inf_{V\in\mathcal{K}} \, \{ f^*(V) - \langle V, \Gamma W \rangle \}
        \\
        &=
        \inf_{W\in\cU,V \in \cK} g(W) + f^*(V) - \langle V, \Gamma W \rangle
        \\
        &=
        \inf_{V\in\mathcal{K}} f^*(V) - \sup_{W\in\cU} \, \{ \langle \Gamma^\sharp V, W \rangle - g(W) \}
        \\
        &=
        \inf_{V\in\mathcal{K}} f^*(V) - g^*(\Gamma^\sharp V).
        \label{eq:dual:final}
    \end{align}
\end{proof}
\subsection{Proof of \texorpdfstring{\Cref{prop:fl_gstar}}{Proposition 3.2}}
We begin by expliciting a formula for $\Gamma^\sharp$.
\begin{lemma}
    Let $H \in \R^{n \times s}$. Then 
    \begin{equation*}
        \Gamma ^\sharp H = \left [\sum_{i=1}^n h_{ij} \phi(x_i) \right ]_{j=1}^s \in \mH^s.
    \end{equation*}
\end{lemma}
\begin{proof}
    Let $W = [w_j]_{j=1}^s \in \mH^s$ and $H \in \R^{n \times s}$. It holds that 
    \begin{align*}
        \langle H, \Gamma W \rangle_{\R^{n \times s}}
        &= \sum_{i=1}^n \sum_{j=1}^s h_{ij} [\Gamma W]_{ij} \\
        &= \sum_{i=1}^n \sum_{j=1}^s h_{ij} \langle \phi(x_i), w_j \rangle_\mH \\
        &= \sum_{j=1}^s \left \langle \sum_{i=1}^n h_{ij} \phi(x_i) , w_j \right \rangle_\mH\\
        &= \left \langle \left [\sum_{i=1}^n h_{ij} \phi(x_i) \right ]_{j=1}^s , W \right \rangle_{\mH^s}
    \end{align*}
    which conludes the proof.
\end{proof}
We are now ready to prove that $g^\star(\Gamma ^\sharp H) = \Tr \left( \sqrt{H^\top G H}\right)$.
\begin{proof}
    Let $W = [w_j]_{j=1}^s \in \mH^s$ and $H \in \R^{n \times s}$.
    \begin{align*}
        \sup_{\mG(W) \curlyeqprec I_s} \langle \Gamma^\sharp H, W \rangle_{\mH^s} 
        &= \sup_{\mG(W) \curlyeqprec I_s} \sum_{j=1}^s \langle w_j, \sum_{i=1}^n h_{ij} \phi(x_i) \rangle_\mH
    \end{align*}
    Since $\mH = \Span \left( \{ \phi(x_i) \}_{i=1}^n \right ) \bigoplus \left ( \Span \left (\{ \phi(x_i) \}_{i=1}^n\right ) \right )^\top $ and using orthogonality properties, we can parameterize $W$ using a matrix $A \in \R^{n \times s}$ so that 
    \begin{equation*}
        \forall j \in [s], w_j = \sum_{i=1}^n a_{ij} \phi(x_i).
    \end{equation*}
    Then 
    \begin{align*}
       \langle \Gamma ^\sharp H, W \rangle_{\mH^s} &= \sum_{j=1}^s \langle \sum_{i=1}^n h_{ij} \phi(x_i) , \sum_{l=1}^n a_{lj} \phi(x_l) \rangle_\mH \\
       &= \Tr A^\top G H.
    \end{align*}
    Moreover, $\mG(W) = A^\top G A$. We thus fall back on the problem
    \begin{align*}
        \sup_{A ^\top G A \curlyeqprec I_s } \Tr A^\top G H.
    \end{align*}
    By performing the change of variable $B = G^{\frac{1}{2}} A$ (possible since $G$ is full rank), the problem becomes 
    \begin{align*}
        \sup_{B ^\top B \curlyeqprec I_s } \Tr B^\top G^{\frac{1}{2}} H.
    \end{align*}
    We recognize the Fenchel-Legendre transform of the indicator of the convex hull of the unit ball for the spectral norm, which is known to be the Schatten $1$-norm (also known as nuclear norm) that we denote $\norm{\cdot}_{S_1}$ in what follows.
    We have thus proven that
    \begin{equation*}
        g^\star (\Gamma ^\sharp H) = \norm{G^{\frac{1}{2}} H}_{S_1}.
    \end{equation*}
    A simple application of the SVD decomposition then suffices to see that
    \begin{equation*}
        \norm{G^{\frac{1}{2}} H}_{S_1} = \Tr \sqrt{H^\top G H},
    \end{equation*}
    which conludes.
\end{proof}
\subsection{Proof of \texorpdfstring{\Cref{prop:gradient_spectral_norm}}{Proposition 3.3}}
\begin{proof}
    We begin by rewriting $\pi(H) = R \circ c (H)$ where $R(X)=\Tr \sqrt{X}$ and $c(H)=H^\top GH$.
	Note that $R(X)=\Tr \sqrt{X} = \sum_{i=1}^s \sqrt{\lambda_i(X)}$, with $\lambda_i(x)$ the $i$-th largest eigenvalue of $X$. 
    Also, $R(X)=r(\lambda(X))$, with $r(z)=\sum_{i=1}^s \sqrt{z_i}$, where $z_i$ indicates the $i$-th component of $z$ and $\lambda(X)$ is the eigenvalues function defined as $\lambda(X)=[\lambda_1(X), \lambda_2(X), \dots, \lambda_s(X)]^\top $. 
    Note that $r$ is symmetric to permutations, so $R$ is a symmetric spectral function over the set of symmetric matrices $\mathbb{S}^{s}$. Therefore, the gradient of $R$ is (see Theorem 1.1 in \cite{lewis1996derivatives}):
	\begin{equation}
		\nabla R(X) = U^\top \text{diag}(\nabla r(\lambda(X))) U,
	\end{equation}
	with $U$ any orthogonal matrix such that $X=U^\top \text{diag}(\lambda(X)) U$. The gradient of $r$ is
	\begin{equation}
		\nabla r(z) = \frac{1}{2\sqrt{z}},
	\end{equation}
	for $z>0$ (which is the case for $G$ positive definite and $H$ full rank). Moreover
	\begin{equation}
		\nabla c(H) = 2GH.
	\end{equation}
	Finally by the chain rule we have
	\begin{equation}
		\begin{split}
			\nabla \pi(H) &= 2GH U^\top \text{diag}(\frac{1}{2\sqrt{\lambda(H^\top GH)}}) U,
		\end{split}
	\end{equation}
	where $U \in \mathbb{R}^{s \times s}$ is any orthogonal matrix satisfying
	\begin{equation}
		H^ \top GH = U^\top \text{diag}\left( \lambda(H^\top GH) \right) U.
	\end{equation}
	Therefore,
	\begin{equation}
		\nabla \pi(H) = GHU^\top DU,
	\end{equation}
	with $D=\text{diag}\left(\frac{1}{\sqrt{\lambda(H^\top GH)}}\right)$.
\end{proof}
\subsection{Proof of \texorpdfstring{\Cref{prop:optimal_svd_dual}}{Proposition 3.5}}
\begin{proof}
Writing the dual cost with $H = U_r^\top \sqrt{\Sigma_r}$, we get 
\begin{align*}
    \frac{1}{2} \Tr(H^\top H) - \Tr(\sqrt(H^\top G H)) &= \frac{1}{2} \Tr(\sqrt{\Sigma_r} U_r U_r^\top \sqrt{\Sigma_r}) - \Tr(\sqrt{\sqrt{\Sigma_r} U_r G U_r^\top \sqrt{\Sigma_r}}) \\
    &= \frac{1}{2} \Tr{(\Sigma_r)} - \Tr{\sqrt{\Sigma_r^2}} \\
    &= - \frac{1}{2} \Tr{(\Sigma_r)}.
\end{align*}
As this quantity correspond to (minus) the primal cost and srong duality holds, then $H$ is a dual solution.
\end{proof}
\subsection{Proof of \texorpdfstring{\Cref{prop:link_critical_reconstruction}}{Proposition 3.6}}
\begin{proof}
    First remark that for all $H \in \R^{n \times s}$, 
    \begin{align*}
        \nabla J (H) = H H^\top H - GH.
    \end{align*}
    Let $\hat{H}$ be a critical point of the dual function, we denote by $\hat{H}^\top G \hat{H} = U^\top S U$ the SVD of $\hat{H}^\top G \hat{H}$.
    The condition $0 \in \partial \left( H \mapsto \frac{1}{2} \norm{H}^2_\fro - \pi(H) \right)(\hat{H})$ implies that 
    \begin{align*}
        \hat{H} = G \hat{H} U^\top S^{-\frac{1}{2}} U 
    \end{align*}
    Thus 
    \begin{align*}
       \hat{H}^\top \hat{H} &=  \hat{H}^\top G \hat{H} U^\top S^{-\frac{1}{2}} U \\
       &= U^\top S U U^\top S^{-\frac{1}{2}} U \\
       &= U^\top S^{\frac{1}{2}} U,
    \end{align*}
    so that
    \begin{equation*}
        \hat{H} \hat{H}^\top \hat{H}  = G \hat{H},
    \end{equation*}
    showing that $\nabla J (\hat{H}) = 0$, which concludes the proof.
    \end{proof}
\subsection{Proof of \texorpdfstring{\Cref{proposition:projections}}{Proposition 3.8}}
\begin{proof}
Let $\hat{H}$ be a solution to the dual problem. We know that the primal solution satisfies 
\begin{equation*}
    \hat{W} \in \argmax_{W \in \mH^s, \mG(W)= I_s} \langle W, \Gamma^\sharp \hat{H} \rangle.
\end{equation*}
We can express each $\hat{w}_j = \sum_{i=1}^n \hat{a}_{ij} \phi(x_i)$, and the problem becomes 
\begin{equation*}
    \max_{A \in \R^{n \times s}} \langle A, G\hat{H} \rangle \quad \st A^\top G A = I_s.
\end{equation*}
By doing the change of variables $B = G^{\frac{1}{2}}A$, we want to solve 
\begin{equation} \label{eq:prob_B}
    \max_{B \in \R^{n \times s}} \langle B, G^{\frac{1}{2}}\hat{H} \rangle \quad \st B^\top B = I_s
\end{equation}
Let $G^{\frac{1}{2}} \hat{H} = V^\top S U$ be the SVD decomposition of $G^{\frac{1}{2}} \hat{H}$. We have that $\hat{B} = V^\top U$ maximizes \ref{eq:prob_B}, which in turn gives $\hat{A} = G^{-\frac{1}{2}} V^\top U$. \\
We can now express $\hat{A}$ using $\hat{H}$, as
\begin{equation*}
    \hat{H} = G^{-\frac{1}{2}} V^\top S U
\end{equation*}
implies that 
\begin{equation*}
    \hat{H} U^\top S^{-1} U = \hat{A}.
\end{equation*}
Remarking that $\hat{H}^\top G \hat{H} = U^\top S^2 U$ is the SVD of $\hat{H}^\top G \hat{H}$ allows to conclude the proof.
\end{proof}
\section{More Experimental Results}
\subsection{Kernel Choice}
Numerical evaluations are conducted to assess the effect of different kernel functions.
In fact, different kernels give different eigenspectra of the Gram matrix. 
In general, numerical algebra solvers' accuracy and required number of iterations may depend on the eigenspectrum of $G$. 
For example, Randomized SVD (RSVD) is known to be less accurate with data matrices whose eigenspectrum decays slowly \cite{halko2011}, which is often the case with real-world noisy data.
On the other hand, in the main body we show that our method is mostly unaffected by the shape of the eigenspectrum of $G$. 

We conduct further experiments to verify that our method is effective when we use other kernels as well. 
The results in Table \ref{tab:speedup:rbf} below show the KPCA training time with Gaussian kernel, while in Table \ref{tab:speedup:high} in the main body the kernel was set as the Laplacian. 
The speedups are clear also with the Gaussian kernel.

\begin{table*}[h]
	\caption{{KPCA Training Time} with Gaussian kernel for $\delta = 10^{-2}$.
	Speedup factor w.r.t. RSVD. 
	}
	\label{tab:speedup:rbf}
	\centering
	\begin{tabular}{lcccccc}
		\toprule
		\multirow{2}{*}{Task} & \multirow{2}{*}{$n$} & \multicolumn{4}{c}{Time (s)}  & \multicolumn{1}{c}{Speedup} \\\cmidrule(lr){3-6}
		& & SVD & Lanczos & RSVD & Ours & Factor\\	
		\midrule
		Synth 1     & 7000  & 87.33  &  1.05  & 3.04   & \textbf{0.56}  & 5.41 \\ 	
		Protein     & 14895 & 979.84 &  3.86  & 10.61  & \textbf{1.08}  & 9.81 \\ 	
		RCV1        & 20242 & -      &  21.03 & 19.69  & \textbf{3.76}  & 5.23 \\ 
		CIFAR-10    & 60000 & -      &  47.94 & 197.99 & \textbf{13.47} & 14.70 \\ 	
		\bottomrule
	\end{tabular}
\end{table*}

Moreover, the results in Table \ref{tab:huber:laplacian} below show the test MSE on the corrupted Iris dataset used in \Cref{sec:experiments:huber} in the main body with Laplacian kernel, while in the main body was set as the Gaussian. 
The Huber loss shows improved robustness also with the Laplacian kernel.

\begin{table}[h]
	\centering
	\caption{MSE on contaminated Iris dataset using Huber loss with Laplacian kernel.}
	\label{tab:huber:laplacian}
	\begin{tabular}{lccc}
		\toprule
		$\tau$ & $\tfrac{1}{2}\norm{\cdot}^2$ & $H_\kappa^2$ & $H_\kappa^1$ \\
		\midrule
		10 & 3.024668 &	2.862165 & 2.762514 \\
		25 & 4.580805 & 4.415508 & 4.337033 \\
		50 & 8.998080 & 8.824152 & 8.808865 \\
		75 & 15.695008 & 15.507481 & 15.590248 \\
		100 & 24.671558 & 24.671039 & 24.671153 \\
		\bottomrule
	\end{tabular}
\end{table}

\subsection{Comparative Analysis}
We conduct comparative experiments between our sparse/robust KPCA with recent sparse/robust KPCA methods. 
First, we compare with $\ell_{2,1}$-RKPCA \cite{wang2020} and with L1-KPCA \cite{kim2020}, which are recent relevant robust KPCA methods from the literature that have been shown experimentally to exhibit good robustness for this task. 
The hyperparameter in \cite{wang2020}, which controls the importance of the $\ell_{2,1}$-norm penalization term, is selected in the range suggested by the authors in their paper. 
In Table \ref{tab:huber:cmp}, we report reconstruction error (MSE) for the Iris data, contaminated as described in \Cref{sec:experiments:huber} in the main body for multiple noise levels. 
Higher MSE means that the learned subspace is influenced by the outliers, while lower MSE corresponds to more robust models.

\begin{table}[h]
    \centering
    \caption{Robust KPCA comparison. MSE on contaminated Iris dataset.}
    \label{tab:huber:cmp}
    \begin{tabular}{lccccc}
		\toprule
        $\tau$ & $\ell_{2,1}$-RKPCA & L1-KPCA & $\tfrac{1}{2}\norm{\cdot}^2$ & $H_\kappa^2$ & $H_\kappa^1$ \\
        \midrule
        10 & 6.965341 & 6.829537 & 7.591059 & 6.833484 & 7.381284 \\
        25 & 7.437400 & 7.598856 & 7.910846 & 7.182663 & 7.687518 \\
        50 & 8.113397 & 8.085395 & 8.691805 & 8.045477 & 8.430957 \\
        75 & 9.734140 & 9.547392 & 9.782740 & 9.259353 & 9.465706 \\
        100 & 10.934794 & 10.855314 & 11.183650 & 10.824293 & 10.791766 \\
        \bottomrule
    \end{tabular}
\end{table}    

Additionally, we compare with SSKPCA \cite{guo2019} and with SKPCA \cite{wang2016}, which are recent relevant sparse KPCA methods from the literature. 
In Table \ref{tab:epsilon:cmp}, the reconstruction error ratio is the ratio between the reconstruction error using the components learned with the corresponding sparse KPCA method and the reconstruction error from the dense problem. 
Ratios closer to 1 are better. 
The settings and dataset are the ones from \Cref{sec:experiments:epsilon} in the main body. 
The sparsity parameter $\lambda$ in SKPCA and $\gamma$ in SSKPCA are selected to obtain a fixed sparsity percentage in the matrix of coefficients.

\begin{table}[h]
    \centering
    \caption{Sparse KPCA comparison. Reconstruction error ratio.}
    \label{tab:epsilon:cmp}
    \begin{tabular}{lcccc}
		\toprule
        Sparsity (\%) & SKPCA  & SSKPCA  & $\ell_\epsilon^2$ &  $\ell_\epsilon^\infty$ \\
        \midrule
        10 & 1.19053 & 1.17019 & 1.04675 & 1.01685 \\
        20 & 1.67440 & 1.18649 & 1.06966 & 1.03824 \\
        30 & 1.83307 & 1.18794 & 1.08153 & 1.04247 \\
        40 & 2.30358 & 1.20480 & 1.09440 & 1.06308 \\
        50 & 2.75248 & 1.22205 & 1.10164 & 1.08127 \\
        \bottomrule
    \end{tabular}
\end{table}        

The robust KPCA in \cite{kim2020} uses L1-norm KPCA, i.e., maximizing variance with respect to the L1-norm, while we consider variance-like objectives with infimal convolution with the squared norm, so we can recover the standard KPCA as a special case. 
The method in \cite{wang2020} introduces extra variables and multiple $\ell_{2,1}$-norm penalizations in the objective and consequently multiple penalty hyperparameters which makes it harder to tune than our proposed method with a single parameter $\kappa$; the optimization is performed with alternating updates with three additional optimization subproblems, which require to compute the SVD of an $n \times s$ matrix at each iteration. 
SKPCA \cite{wang2016} and SSKPCA \cite{guo2019} relax and modify KPCA to use ElasticNet optimization for promoting sparsity. To get sparsity, SKPCA make use of the matrix $G^{-1/2}$, which requires computing the full SVD of $G$, and SSKPCA's iterative algorithm is initialized with the top $s$ eigenvectors of $G$, requiring to compute the truncated SVD of $G$. 
In comparison, we get sparsity without performing such costly operation, as our modelling considers the SVD of the much smaller $s \times s$ matrix $H^\top GH$.
Overall, our formulation performs similarly without having to compute the full SVD.
\end{document}